\documentclass[sigconf]{acmart}
\usepackage{amsmath,amsfonts, amsthm}
\usepackage{algorithmic}
\usepackage{enumitem}
\usepackage[linesnumbered,ruled,vlined]{algorithm2e}
\usepackage{graphicx}
\usepackage{textcomp}
\usepackage{xcolor}
\usepackage{soul,color}
\usepackage{multirow}
\usepackage{array}
\usepackage{url}
\usepackage{booktabs}
\usepackage{subcaption}
\usepackage{mathrsfs}
\usepackage{wrapfig}

\newtheorem{prop}{Proposition}

\newcommand{\fdl}{federated learning\ }

\newcommand{\mini}{\textit{mini}ImageNet}

\newcommand{\etal}{\textit{et al.}}

\newcommand{\Ttsupsubeq}[2]{[\Theta^{#1}_{#2},\theta^{#1}_{#2}]}

\newcommand{\wsupsubeq}[2]{w^{#1}_{#2}}

\newcommand{\gradsubeq}[1]{\nabla_{#1}}

\newcommand{\Tcaleq}{\mathcal{T}}

\AtBeginDocument{%
  \providecommand\BibTeX{{%
    \normalfont B\kern-0.5em{\scshape i\kern-0.25em b}\kern-0.8em\TeX}}}

\setcopyright{acmcopyright}
\copyrightyear{2021}
\acmYear{2021}
\acmDOI{10.1145/1122445.1122456}

\acmConference[Arxiv]{}
\acmBooktitle{submission}
\acmPrice{15.00}
\acmISBN{978-1-4503-XXXX-X/18/06}

\begin{document}

\title{Federated Few-Shot Learning with Adversarial Learning}

\author{Chenyou Fan}
\email{fanchenyou@cuhk.edu.cn}
\affiliation{%
  \institution{Shenzhen Institute of Artificial Intelligence and Robotics for Society}
  \country{China}
}

\author{Jianwei Huang}
\email{jianweihuang@cuhk.edu.cn}
\affiliation{%
  \institution{The Chinese University of Hong Kong, Shenzhen}
  \country{China}
}


\begin{abstract}
We are interested in developing a unified machine learning model over many mobile devices for practical learning tasks, where each device only has very few training data. This is a commonly encountered situation in mobile computing scenarios, where data is scarce and distributed while the tasks are distinct. In this paper, we propose a federated few-shot learning (FedFSL) framework to learn a few-shot classification model that can classify unseen data classes with only a few labeled samples. With the federated learning strategy, FedFSL can utilize many data sources while keeping data privacy and communication efficiency. There are two technical challenges: 1) directly using the existing federated learning approach may lead to misaligned decision boundaries produced by client models, and 2) constraining the decision boundaries to be similar over clients would overfit to training tasks but not adapt well to unseen tasks. To address these issues, we propose to regularize local updates by minimizing the divergence of client models. We also formulate the training in an adversarial fashion and optimize the client models to produce a discriminative feature space that can better represent unseen data samples. We demonstrate the intuitions and conduct experiments to show our approaches outperform baselines by more than 10\% in learning vision tasks and 5\% in language tasks.
\end{abstract}




\maketitle

\section{Introduction}

\begin{figure}
\begin{center}
\includegraphics[clip, trim=0 0 10 10, width=0.42\textwidth]{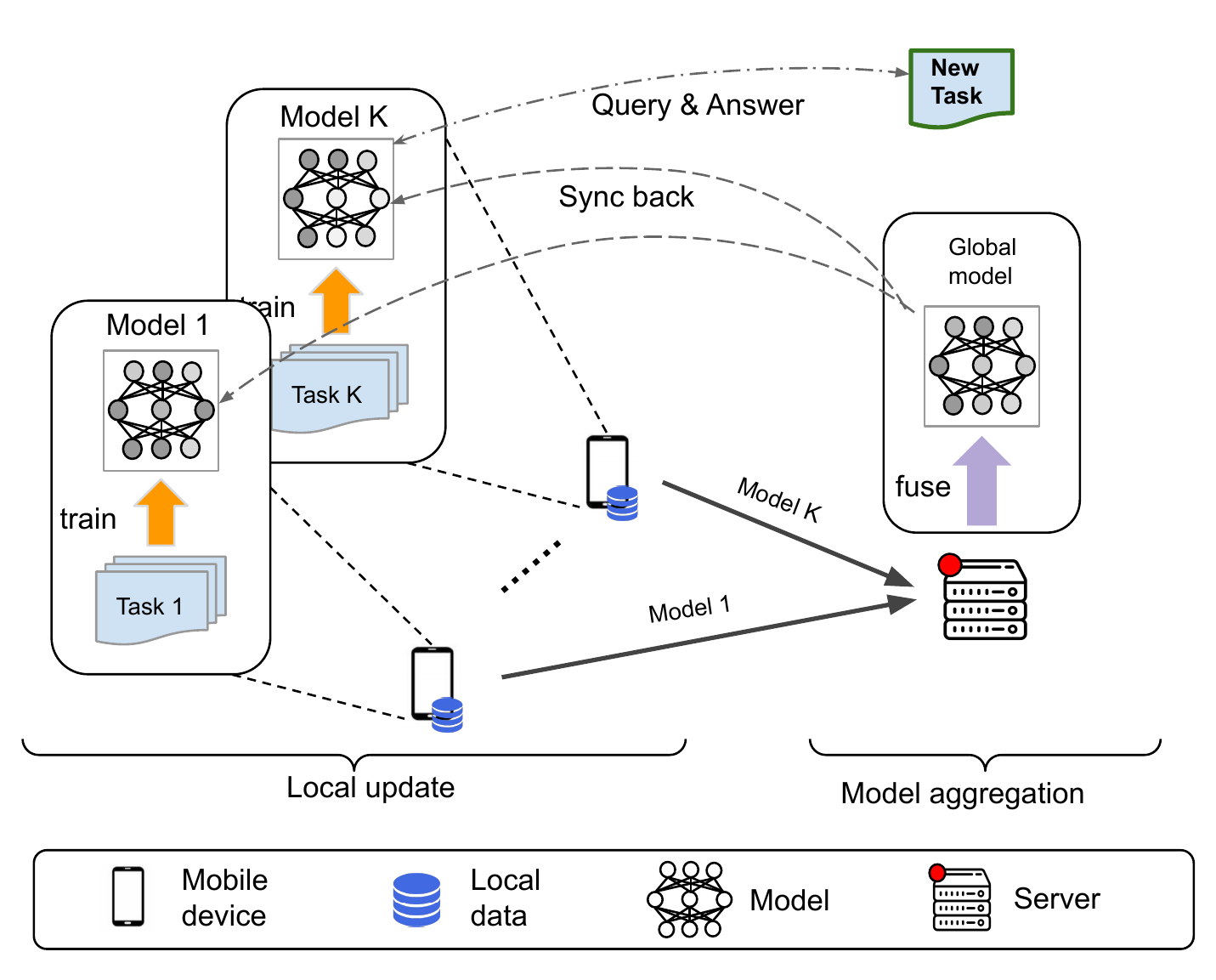}
\end{center}
\vspace{-10pt}
\caption{Overview of a FedFSL system. Distributed client devices train client models with sampled FSL tasks from local data. Then client models are sent to the central server and aggregated to a central model, which is sent back to clients for the next round of local updates.
}\label{fig:fl_fsl_demo}
\vspace{-10pt}
\end{figure}

In recent years, mobile devices such as smartphones and tablet computers have become perhaps the most frequent and convenient way of connecting users to the Internet. A capable mobile device can now get access to first-hand user data such as photos, voice, typing records, and so on, thus it becomes an ideal platform to deploy machine learning models to better understand user intents and assist in completing daily tasks.
% e.g., a pre-trained language model can help users auto-complete a phrase when replying a short message or an Email.

%However, training models on mobile devices is a challenging task. 
Conventional distributed machine learning approaches~\cite{li2014scaling, zhang2015deep} require the data to be transferred from clients to a central server, which raises serious concerns of data privacy.
A recent proposed approach to address this issue is \fdl (FedL)~\cite{fedavg,zhao2018federated,li2020federated}. In the FedL paradigm, each participating client computes a local machine learning model with its own data, while a central server periodically coordinates client models by model aggregation (without collecting the actual data). FedL provides a promising way of learning models on distributed devices while preserving data privacy and locality.

However, existing FedL approaches assume each participating client has sufficient training data for the tasks of interest. The realistic situation is that data collected with mobile devices can be insufficient and almost always unlabeled. This severely limits the practicality and scalability of FedL in realistic applications.
For example, image classification with federated deep neural networks (DNNs) is a commonly compared task in recent FedL studies~\cite{fedavg,zhao2018federated,li2020federated}. Training such a DNN in supervised fashion requires hundreds of labeled training samples or more for every class. In reality, each mobile user may own just one or a few samples of each interested image category, and the mobile user often does not have time or interest to label these images. The huge gap between lab scenarios with abundant labeled data and real situations with scarce and mostly unlabeled data motivates us to consider the following question: \textbf{How to make FedL effective in data-scarce scenarios?}

\begin{table*}[htb]
\centering
\small
\begin{tabular}{l|ccccccc} \toprule  
\hspace{-5pt}
Method & Federated learning & Non-convex model & SGD solver  & Few-shot learning & Divergence &  Feature learning \\ \toprule
FedAvg\cite{fedavg} / FedProx\cite{li2018federated}   & \checkmark   & \checkmark & \checkmark         \\
FEDL~\cite{dinh2019federated}  & \checkmark  & & \checkmark  \\
% One-shot FL~\cite{guha2019one} & \checkmark & & &  \\ 
% Distillation~\cite{hinton2015distilling} &&\checkmark &\checkmark & & \checkmark \\
Mutual learning~\cite{zhang2018mutual}  & &\checkmark &\checkmark & &\checkmark                \\
Local adaptation~\cite{yu2020salvaging}  &\checkmark &\checkmark &\checkmark & &\checkmark    \\
Gradient-based FSL\cite{finn2017model} & &\checkmark &\checkmark & \checkmark    \\  
Metric-based  FSL~\cite{Gidaris_2018_CVPR} & &\checkmark &\checkmark & \checkmark &   & \checkmark\\ 
\midrule 
FedFSL-naive (\cite{chen2018fedmeta}, \ref{sec:fed_maml})    &\checkmark &\checkmark &\checkmark & \checkmark    \\
\textbf{FedFSL-MI} (ours, \ref{sec:fed_maml_mi})  &\textbf{\checkmark} &\textbf{\checkmark} &\textbf{\checkmark} &\textbf{\checkmark} & \textbf{\checkmark} &       \\ 
\textbf{FedFSL-MI-Adv} (ours, \ref{sec:adv})  &\textbf{\checkmark} &\textbf{\checkmark} &\textbf{\checkmark} &\textbf{\checkmark} & \textbf{\checkmark} & \textbf{\checkmark}    \\ \bottomrule
\end{tabular}
\vspace{1pt}
\caption{Comparisons of relevant approaches with supported features.}
\vspace{-15pt}
\label{tab:checklist}
\end{table*}

A recently developed approach to address the issue of insufficient training data is few-shot learning (FSL)~\cite{vinyals2016matching, finn2017model, li2017meta}. FSL aims to develop machine learning models to solve unseen tasks with very few labeled training data, but often in the context of a single data source.
In this paper, we propose a federated few-shot learning (FedFSL) framework for efficiently training few-shot classification models from many data sources. The few-shot model can tackle novel tasks with just a few labeled data for unseen classes. As shown in Fig.~\ref{fig:fl_fsl_demo}, the paradigm of FedFSL is to first perform local updates with few-shot tasks sampled from local data, then send local models to a central server for model aggregation and coordination.
Through inheriting the merits of FedL, FedFSL also preserves the data privacy and communication efficiency.

FedFSL has many potential applications for utilizing machine learning models on mobile devices. For example, a few-shot language model can be used to suggest words by learning from just a few typing records from each of many users; a few-shot sentiment analysis model can be used to depict user profiles given only a few tweets posted by each of the many users; a few-shot face recognition model can identify users and their friends by learning from just few annotated photos by each of many mobile users.
%media recommendation model can be used to suggest media contents such as movies and songs
%by learning from only a few browsing records of many users.

There are two technical challenges we encounter during the development of the FedFSL framework: 1) directly using the existing FedL approaches to the data-scarce scenarios may lead to misaligned decision boundaries produced by client models, and 2) constraining the decision boundaries to be similar over clients would develop a classifier overfit to training tasks but not transferable to unseen tasks. To address these issues, we propose to regularize local updates by minimizing the divergence between client models and the central model. We then formulate the training of the feature generator part and classifier part of the model in adversarial fashion. In this way, the client models are explicitly optimized to produce a discriminative feature space that can better represent unseen data samples. We demonstrate the intuitions and conduct experiments to show the effectiveness of our approaches.

Our contributions can be summarized as follows:
\begin{itemize}[leftmargin=*]
    \item We propose a novel federated few-shot learning framework (FedFSL) that can perform effective federated learning on few-shot tasks. This represents the first step in addressing the commonly encountered but overlooked scenarios in mobile computing where training data is scarce and testing tasks are distinct.
    \item We present the key innovations in formulating FSL with federated clients as well as explicitly optimizing the federated model by minimizing model discrepancy in challenging non-IID scenarios.
    \item We define a novel concept of mutual divergence of federated client models, which can be minimized to better coordinate the client training on scarce local data. 
    \item We design a dedicated adversarial learning approach to construct a discriminative feature space, which better generalizes to unseen tasks compared with existing training procedures of FSL models. 
     \item We evaluate our framework by modelling different types of structured data (such as images and sentences) with both CNN and RNN models, showing its effectiveness and practical usability in modelling various learning tasks in machine vision and NLP.
    \item Our approaches significantly outperform baselines that are either non-distributed or not aligning the feature space across the clients by more than 10\% on vision tasks and 5\% on language tasks.
    % by 10\%$\sim$20\% on benchmark datasets. Also, our model coordination technique boosts the performance by 5\%$\sim$10\%.
    %\tred{result} 
\end{itemize}

\section{Related work}

We will briefly review recent related work 
in two categories: (i) studies
either \fdl (FedL) (e.g.,~\cite{fedavg,zhao2018federated,li2018federated,li2020federated,dinh2019federated,zhang2020fedpart}) or few-shot learning (FSL) (e.g.,~\cite{vinyals2016matching, snell2017prototypical,munkhdalai2017meta,li2017meta,sun2019mtl,Gidaris_2018_CVPR, sung2018learning, Li_2020_adap_margin}), or both of them~\cite{chen2018fedmeta} (ii) studies proposing similar ideas of minimizing model divergence to better learn individual models or an ensemble model.

% In this section, we will briefly review the related work that either relates to both \fdl (FedL) has become a rapidly developing topic in the research community~\cite{fedavg,zhao2018federated,li2018federated,li2020federated,dinh2019federated,zhang2020fedpart}, as it provides a new way of learning models over a collection of distributed mobile devices while still preserving data privacy. 

% Few-shot learning (FSL)~\cite{vinyals2016matching, snell2017prototypical} is another hot topic in machine learning field
% which learns to recognize novel classes with very few labeled training samples. Recent state-of-the-art approaches either focus on learning a distance metric that can distinguish categories ~\cite{Gidaris_2018_CVPR, sung2018learning, Li_2020_adap_margin} or 
% searching for a good initialization of the model parameters 
% which can adapt to novel classification tasks with only a few gradient update steps~\cite{finn2017model,munkhdalai2017meta,li2017meta,sun2019mtl}.
% focus on learning a good metric in feature space such that data samples of same classes have lower distances while samples of different classes have larger distances.

% Li~\etal~\cite{li2020diff} proposed a differentially private algorithm for securing parameter transfer across devices or learning stages. Though they consider FL and FSL as two separated applications of their technique, their goal is to secure data privacy during model sharing instead of performing FSL with federated devices.

To our best knowledge, training FSL models on distributed devices is still an under-explored open problem. 
The first work of this topic was from
Chen~\etal~\cite{chen2018fedmeta} who explored federated meta-learning by applying FedAvg on meta-learning approaches such as MAML~\cite{finn2017model} in a straightforward way. However, their goal is to improve supervised learning by better sharing models among
%which can learn on unseen tasks (i.e., transferable knowledge) in
federated clients, instead of learning few-shot tasks.
% Hereby, they evaluated their models with conventional supervised learning settings instead of few-shot setting, i.e, they
%partitioned data by different users while trained and evaluated the models on different partitions. 
They neither evaluated their models on FSL tasks, nor explicitly considered dealing with the underlying data heterogeneity in different devices (e.g., non-IID case) which can severely harm FSL.
Our model captures the idea of federating transferable knowledge
among distributed clients. We further explore the practical data-scarce scenarios and evaluate our models on challenging benchmark FSL datasets.
In addition, we explicitly resolve the data heterogeneity issue by proposing a family of more effective meta-learning approaches designed for federated settings.

The other work that loosely connects FSL with FedL
is Li~\etal~\cite{li2020diff}, which proposed a differentially private algorithm for securing parameter transfer across devices or learning stages. The authors considered FedL and FSL as two separated applications of their technique, and their goal is to secure data privacy during model sharing instead of performing FSL with federated devices.

% Directly applying existing FL approaches (e.g., FedAvg~\cite{fedavg}) on FSL would be less capable as FSL procedures depend heavily on locally sampled tasks which could be distinct due to data heterogeneity. To coordinate the training of client models, we propose to minimize the divergence of predictions between each client and the central model. 

There are two recent studies relating to the techniques that we utilize to coordinate client model training and to learn the consistent feature space, though neither of them considered few-shot learning.
Yu~\etal~\cite{yu2020salvaging} used knowledge distillation~\cite{hinton2015distilling} (a variant of KL-divergence) in federated learning to learn different client models to better fit local data. 
%Their goal was to use a teacher model to better train student models. 
Our approach is different from theirs such that 1) they did not perform model aggregation as they want to train client models that fit on local data, but our method performs aggregation for learning a unified global model, 2) their method followed the normal supervised learning paradigm which cannot solve few-shot tasks, and 3) they did not consider feature space learning but our method does.
Zhang~\etal~\cite{zhang2018mutual} proposed to minimize the KL-divergence of every client model pair to enhance ensemble learning. However, this imposes heavy computation and communication costs in distributed scenarios. On the contrary, we propose to approximate the client pairwise KL-divergence with the divergence between the client model and the federated global model, which integrates into FedL seamlessly.

It's also worth to clarify that several recent studies~\cite{guha2019one, salehkaleybar2019one, shin2020xor} 
focused on reducing the communications of distributed learning with one or a few communication rounds under federated learning settings. Though this topic is also referred as ``few-shot federated learning", it's totally different from our work which studies federated learning under data-scarce scenarios.

\section{Federated Few-shot Learning}
\label{sec:approach}

In this section, we formulate the Federated Few-Shot Learning (FedFSL) framework in details. Specifically, we will study the few-shot classification task which learns to classify on novel classes with few training samples.
We will first review the general federated learning (FedL) objective in~\ref{sec:fl} and the general centralized few-shot learning (FSL) procedures in~\ref{sec:fsl} respectively, based on which we propose the FedFSL formally in
\ref{sec:fed_maml}.

% We also have a novel class set $C_n$ consisting of $n_n$ new classes with non-overlapping with base classes, and each class only has a few labeled data points. Fed-FSL aims to learn a good classifier for novel classes.
% We define a few-shot learning task with its conventional terminology: a N-way P-shot classification task will sample K labeled data from each of N classes which gives a total of $N\cdot K$ samples to perform N-way classification.
% We show the frequently used notations in Table~\ref{tab:freq_notation} for reference.
 
\subsection{General Federated Learning Objective}
\label{sec:fl}
First, we briefly review federated learning (FedL) paradigm and its common implementations in this subsection. 
We consider a distributed system of $K$ clients, each owning a local data source. In FedL, each client trains a local machine learning model based on its local data, while a central server coordinates the clients periodically by collecting their parameters, aggregating them into a central model and sending its parameter back to all clients.
% Federated learning (FL) is a new distributed machine learning paradigm which preserves data at local sources while exchanges models weights with communications between clients and the server for coordination. In particular, FedAvg~\cite{fedavg}, the first and perhaps the most concise form of FL, 
% proposed to update each client model in parallel at first, then aggregate client models by weight averaging to a central model, at last push back the central model to each client for coordination. 
This process will repeat for multiple communication rounds until convergence or timeout.  

Formally, let 
$n_k$ be number of data samples of client $k$, $n=\sum_k n_k$ be total samples across the devices, $w$ be the learning model.
We consider a local objective for client $k$ as the average loss over all data samples
\begin{equation}
\small
\mathcal{L}_k(w)=\frac{1}{n_k} \sum_{i=1}^{n_k} f(x_i,y_i;w),
\label{eq:fed_loss}
\end{equation}
in which $f$ is a loss function that evaluates the prediction of model $w$ on a data sample $(x_i, y_i)$. 
The type of loss function $f$ depends on the task and is known by all clients.
For example, in a classification task with deep neural networks (DNNs), $f$ is often chosen as the cross-entropy loss applied on the models' probabilistic outputs.

The global target is a weighted average of local objectives
\begin{equation}
\small
\underset{w}{\text{min}} \ \mathcal{L}(w) = \sum_{k=1}^{K} p_k \mathcal{L}_k(w),
\label{eq:fed_def}
\end{equation}
in which $p_k = n_k/n$. 
However, as direct data exchange is prohibited in FedL, directly optimizing the global objective  \eqref{eq:fed_def} needs to perform a 
full batch gradient descent on all data that each client holds, and perform 
model aggregation after each client update. This requires high memory usage (if the client model is large) and excessive communications for exchanging models between clients and the server.
A common practice to resolve this issue is to approximate the global objective, such as in the
FedAvg~\cite{fedavg} algorithm which optimizes each local objective $\mathcal{L}_k$ individually and in the FedProx~\cite{li2018federated} which solves the local objectives with proximal terms to regularize training.

% Two widely used approximations are as follows:
% \begin{itemize}
% \item FedAvg~\cite{fedavg} proposes to optimize each local objectives $\mathcal{L}_k$ individually with local data to obtain local optimal weights $w_k$. Then a central server aggregates the client models by averaging to approximate the global solution 
% \begin{equation}
% w = \sum_{k=1}^K p_k w_k
% \end{equation}
% After that the central model is sent back to clients to complete a communication round.
% \item FedProx~\cite{li2018federated} is a variant of FedAvg by adding a proximal term in local objective such that 
% \begin{equation}
% \mathcal{L}^{prox}_k(w)=\mathcal{L}_k(w) + \mu ||w-\bar{w}||^2
% \label{eq:fed_prox}
% \end{equation}
% %$\underset{w}{\text{min}} \ \mathcal{L}(w) + \mu ||w-w^t||^2$ 
% in which $\bar{w}$ is the synchronized global model. The proximal term aims to regularize local updates to be closer to the global model to avoid divergence.
% \end{itemize}

Existing FedL approaches often assume that the clients
always hold sufficient training data for a same task, e.g., all the clients should own enough data samples of the same categories in a classification task.
However, the realistic situation is each client may own a few labeled data samples for certain categories for training, and may encounter unlabeled data samples for testing with unseen true categories. This leads us to study the few-shot classification task which learns to classify on novel classes with few training samples in the following sections.

% For example, 
% many recent FedL studies~\cite{fedavg, zhao2018federated} compared on the image classification task 
% the classes of images that can be classified are limited by the dataset and each class is assumed to have enough training samples. 
% How to recognize unseen data classes with few labeled training data is still an open problem.

\subsection{General Centralized Few-Shot Learning}
\label{sec:fsl}
Next, we briefly review centralized few-shot learning (FSL) procedures. FSL aims to learn a generic model which can adapt to 
unseen tasks with only a few labeled training samples. In this paper, we study the few-shot classification task which aims to classify novel classes.  We define an $N$-way $P$-shot $Q$-query FSL task as a task of training a model with $P$ labeled images for each of $N$ classes and then
evaluating the model with $Q$ unlabeled query images for each class. $P$ is typically very small such as 1 or 5 as ``few-shot" implies.

Let us consider a toy example of classifying animal pictures. The training data are images of cats and dogs which we call \textit{base classes}. The testing data are images of tigers and wolves which we call \textit{novel classes}; each novel class owns one labeled image and many unlabeled images.
We wish to develop a model trained with base classes (cat and dog) that can predict on tiger and wolf (2-way) samples by observing each category just one labeled image (1-shot), i.e., a 2-way 1-shot FSL task.

\begin{figure}
\begin{center}
\includegraphics[clip, trim=0 0 0 0, width=0.46\textwidth]{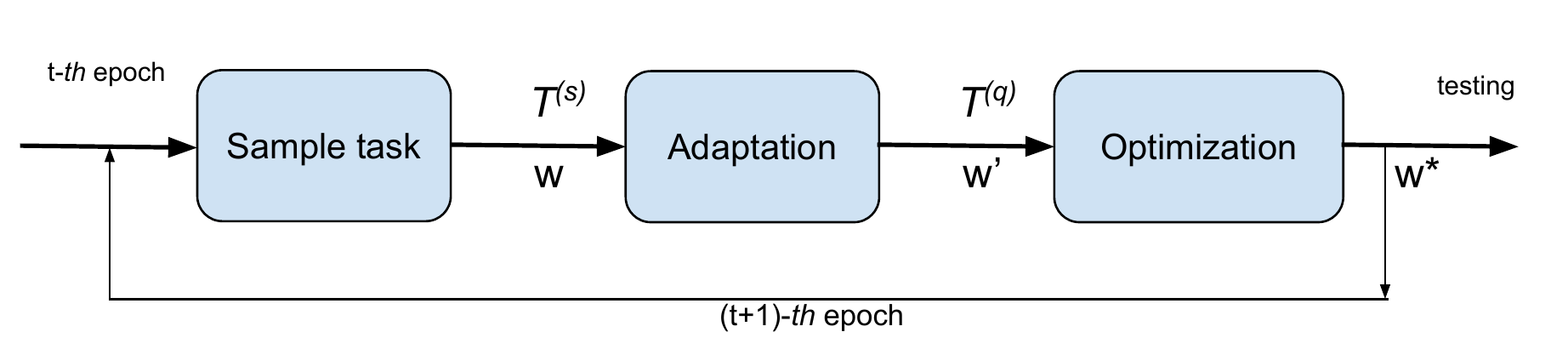}
\end{center}
\vspace{-10pt}
\caption{Three-step meta-learning of FSL. 
}\label{fig:fsl_1}
\vspace{-12pt}
\end{figure}

To train a capable FSL model, recent state-of-the-art work~\cite{Gidaris_2018_CVPR, sung2018learning, finn2017model, li2017meta} adopted a training strategy called meta-learning which samples various few-shot tasks from training data and optimizes the model to fast adapt to these new tasks. The key idea of meta-learning is to learn some transferable knowledge from few data samples in training data that can apply on unseen data. 
State-of-the-art gradient-based approaches~\cite{finn2017model,sun2019mtl} consider a good model initialization as transferable knowledge if it could adapt to various few-shot tasks with just a few gradient steps. They propose to explicitly search for such an initial model as follows. 

The training objective is to minimize the training loss over a batch of tasks $\Tcaleq \in \mathcal{B}$ as follows
\begin{equation}
\begin{split}
w^* &= \underset{w}{\text{min}} \ \mathcal{L}(w) = \frac{1}{|\mathcal{B}|} \sum_{\Tcaleq \in \mathcal{B}} \ell_\mathcal{T}(w), \\
\end{split}
\label{eq:fed_def_adp}
\end{equation}
in which $w^*$ is the optimized model that trains to fast adapt to new tasks and $\ell$ is a task loss. This can be tackled with an iterative approach in which each iteration can be decomposed into three steps, as shown in Fig.~\ref{fig:fsl_1}.
\begin{itemize}[leftmargin=*]
\item \textbf{Sampling step:} The first step is to sample a few-shot task 
$\mathcal{T}$,  also called an \textit{episode}, from base classes. 
For an $N$-way $P$-shot $Q$-query few-shot task, an episode consists of $P$ data instances sampled from each of $N$ distinct base classes as a support set $\mathcal{T}^{(s)}$, and $Q$ data instances sampled from the same $N$ classes as a query set $\mathcal{T}^{(q)}$, which gives a total of $(P+Q) \cdot N$ instances.
\item \textbf{Adaptation step:} The second step is to adapt the current model to the sampled task with gradient descents. This step uses the few labeled data in the support set  $\mathcal{T}^{(s)}$ and performs one or several gradient steps towards optimizing the model weights to the sampled task such that
\begin{equation}\label{eq:adp}
\begin{split}
\wsupsubeq{'}{} = \wsupsubeq{}{}-\alpha \gradsubeq{w}f_{\mathcal{T}^{(s)}} (w),
% \wsupsubeq{'}{k} &= w-\alpha \gradsubeq{w} \mathcal{L}_k
% (w) \\
% &= w-\alpha \gradsubeq{w} \frac{1}{|\mathcal{B}_k|} \sum_{\mathcal{T}_k} f_{\mathcal{T}^{(s)}_k}(w)
\end{split}
\end{equation}
in which $w'$ is the adapted model, $\alpha$ is the step size. %$\gradsubeq{w}$ is the gradient of the loss function. 
% Eq.~\ref{adp} alters the initial parameters in the direction that best suited to the given task.
\item \textbf{Optimization step:} The final step is to evaluate $w'$ with more samples in the query set $\mathcal{T}^{(q)}$ with the empirical loss function
\begin{equation}
\begin{split}
\ell_\mathcal{T}(w) & =  f_{\Tcaleq^{(q)}}(w')  = f_{\Tcaleq^{(q)}}(w-\alpha \nabla f_{\mathcal{T}^{(s)}}(w)), \\
\end{split}
\label{eq:fed_def_adp_l1}
\end{equation}
which can be 
solved by another gradient descent
\begin{equation}
\begin{split}
% \mathcal{L}(w) & = \frac{1}{|\mathcal{B}|} \sum_{\Tcaleq} \ell_\mathcal{T}(w) \\
% w^* &= \underset{w}{\text{min}} \ \mathcal{L}(w)
w \leftarrow w -\beta \ \gradsubeq{w} \ell_\mathcal{T}(w),
\end{split}
\label{eq:fsl_sgd}
\end{equation}
in which $\beta$ is the learning rate.
\end{itemize}
The above procedures are summarized in Fig.~\ref{fig:fsl_1}.
For a centralized FSL to converge, the sampling-adaptation-optimization procedures are repeated for many iterations. This will produce the optimal parameter that best adapts to few-shot tasks. 

During inference, a learned FSL model firstly adapts to unseen tasks with a few labeled samples with \eqref{eq:adp}, then predicts the labels on query samples. In our toy example mentioned above, the model is trained with cat and dog samples to discern two species, then is used to classify tiger and wolf samples with its learned capacity of distinguishing patterns. 
% . Similar to training stage, each task is composed of a support set which is used to adapt the model to new task, as well as a query set to evaluate the model with standard metrics such as classification accuracy and cross-entropy loss.

% Task distributions over client data are defined as $p(\mathcal{T}^1),\dots,p(\mathcal{T}^K)$.
\begin{figure}
\begin{center}
\includegraphics[clip, trim=0 0 0 0, width=0.4\textwidth]{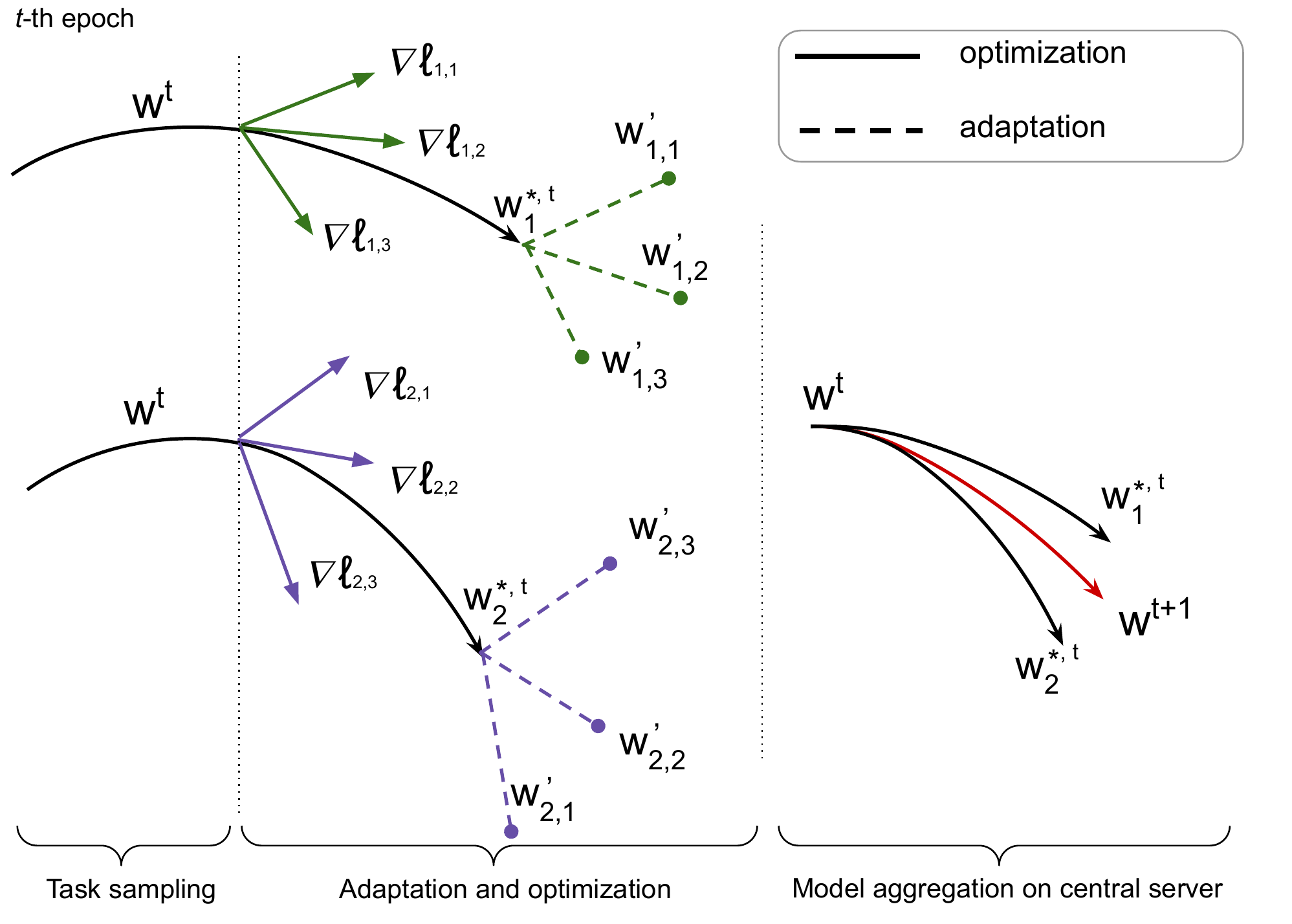}
\end{center}
\vspace{-10pt}
\caption{Demo of a two-client case of FedFSL based on meta-learning procedure.
}\label{fig:fed_fsl_2_client}
\vspace{-10pt}
\end{figure}

\subsection{Federated Few-shot Learning (FedFSL)}
\label{sec:fed_maml}
% In this section, we will 
% We wish to develop a few-shot learning model on mobile devices for facilitating practical applications with machine learning techniques.
% In the meantime, it's also important to keep data privacy and communication efficiency.
As our purpose is to facilitate distributed devices to learn models for few-shot tasks,  we need study how to design such a framework in which meta-learning procedures can be integrated in the federated learning.
We propose Federated Few-shot Learning (FedFSL) in this section. The goal of FedFSL is to search for a \textit{global} optimal model $w^*$ learned on distributed data sources that can best perform few-shot tasks.

Suppose we have $K$ participating clients and each of them can sample batches of few-shot tasks $\mathcal{T}_k \in \mathcal{B}_k$ from their local data sources as discussed in previous section. We first define $k$-th client's local FSL objective by extending from \eqref{eq:fed_def_adp}
\begin{equation}
\begin{split}
\mathcal{L}_k(w) & = \frac{1}{|\mathcal{B}_k|} \sum_{\Tcaleq_k \in \mathcal{B}_k} \ell_{\mathcal{T}_k}(w), \\
\end{split}
\label{eq:fed_def_adp_k}
\end{equation}
in which the subscript $k$ of $\Tcaleq_k$ emphasizes that it's sampled from local data source of client $k$.
% with the difference that the centralized dataset is now replaced with distributed data sources.

Our goal is to find a global optimal model $w^*$ which minimizes the weighted average of local FSL objectives. We formulate this global target as an FedL problem similar to \eqref{eq:fed_def}, i.e., 
\begin{equation}
\begin{split}
w^* &= \underset{w}{\text{min}} \ \mathcal{L}(w) = \sum_{k=1}^{K} p_k \mathcal{L}_k(w) = \sum_{k=1}^{K} \frac{|\mathcal{B}_k|}{|\mathcal{B}|} 
 \mathcal{L}_k(w) \\
\end{split}
\label{eq:fed_def_all_0}
\end{equation}
% \begin{equation}
% \begin{split}
% \underset{w}{\text{min}} \ \mathcal{L}(w) &= \sum_{k=1}^{N} p_k \mathcal{L}_k(w) \\
% & = \sum_{k=1}^{N} p_k \sum_{\Tcaleq_k} \mathcal{L}_{\Tcaleq^{(q)}_k} (f_{ \wsupsubeq{}{}-\alpha \gradsubeq{\wsupsubeq{}{}}\mathcal{L}_{\mathcal{T}^{(s)}_k} (f_{\wsupsubeq{}{}})})
% \end{split}
% \label{eq:fed_def_all_0}
% \end{equation}
Directly optimizing \eqref{eq:fed_def_all_0} would be difficult, as learning an optimal $w^*$ for all distributed clients would require excessive communications across clients, as discussed in \ref{sec:fl}. We thus provide an efficient algorithm for tackling this issue.

Motivated by FedAvg~\cite{fedavg}, we propose a straightforward way of solving a surrogate objective of \eqref{eq:fed_def_all_0} to approximate the global solution, which we call \textbf{FedFSL-naive}.
As shown in Fig.~\ref{fig:fed_fsl_2_client},
FedFSL-naive iteratively updates the central model $w$ by (i) first optimizing each local objective of \eqref{eq:fed_def_adp_k} in parallel, and (ii) aggregating local models to the central model, which update the global model and send it back to clients for the next round of optimization. Formally,
\begin{itemize}[leftmargin=*]
\item At the $t$-th optimization round, each client $k$ optimizes the following local objective
\begin{equation}
\begin{split}
w^{*,t}_k &= \underset{\wsupsubeq{}{}}{\text{argmin}} \ \mathcal{L}_k(\wsupsubeq{}{}) = \underset{\wsupsubeq{}{}}{\text{argmin}} \ \frac{1}{|\mathcal{B}_k|} \sum_{\Tcaleq_k \in \mathcal{B}_k} \ell_{\Tcaleq_k} (w), \\
\end{split}
\label{eq:fed_maml}
\end{equation}
in which the FSL loss $\ell(w)$ is given by \eqref{eq:fed_def_adp_l1}. Fig.~\ref{fig:fed_fsl_2_client} shows a two-client example, in which each client updates on three sampled tasks with \eqref{eq:fed_maml} and obtains local optimal models $w^{*,t}_1$ and $w^{*,t}_2$. The clients then send these local parameters to the central server. 
\item Then the central server approximates the optimal global solution by averaging the client models such that
\begin{equation}
\label{eq:fuse}
w^{t+1} = \sum_{k=1}^C  \frac{|\mathcal{B}_k|}{|\mathcal{B}|} \wsupsubeq{*,t}{k},
\end{equation}
which will be synchronized to all clients for next round of optimization.
\end{itemize}
The above steps \eqref{eq:fed_maml} and \eqref{eq:fuse} are repeated for multiple rounds until convergence. We summarize the procedures in Algorithm~\ref{algo:fed_maml}.

{\SetAlgoNoLine
\begin{algorithm}[htpb]
\small
\DontPrintSemicolon
\LinesNumberedHidden
\KwIn{A set of $K$ federated clients. A local FSL objective $\mathcal{L}_k$ for each client $k$.
% The total dataset $\mathcal{D}$ is the union set of all actual local data $\{D_1,\dots,D_K\}$ on $N$ clients. 
% Task distributions $p(\mathcal{T}^1),\dots,p(\mathcal{T}^K)$ of client data could be identical or not. 
% Learning rate $\alpha$, $\beta$. 
}
\KwOut{A global model $\wsupsubeq{}{}$ optimized for FSL task.}
\textbf{Server executes:} \;
\Indp Initialize global model $w^0$ \;
$t \leftarrow 1$ \;
% Randomly initialize $\wsupsubeq{}{k}$ for all clients. \;
%\For{\textup{each round} $t = 1,2,\dots,T$ } {
\While{t $\leq$ maximum rounds $T$}{
    \For{\textup{each client} $k$ \textup{\textbf{in parallel}}}{ 
    $w_k^{t} \leftarrow $ ClientUpdate{$(w^t)$} 
    }
    Clients send model parameters $w_{1...K}^{t}$ to server \;
    $w^{t+1} \leftarrow \sum_{k=1}^K \frac{|\mathcal{B}_k|}{|\mathcal{B}|} w^{t}_k$  \tcp*[l]{model avg}
    The server sends $w^{t+1}$ back to clients \;
    $t \leftarrow t+1$
 }
Return $w^t$ \;
\;
\Indm \textbf{ClientUpdate}$(w)$: \;
\Indp \KwIn{global model from previous round $w^t$}
\KwOut{updated local model $w^t_k$}
%  $\mathcal{B}_k \leftarrow$ (split local data into batches of episodes) \;
 $w \leftarrow w^t$ \;
 Sample a batch of episodes $\mathcal{B}_k=\{\mathcal{T}_1,...,\mathcal{T}_n\}$ \;
 $w^k_t \leftarrow $ Solve Eq.\eqref{eq:fed_maml} with SGD \;
 Return $w^t_k$ \;
% \For{\textup{each episode} $\mathcal{T} \in \mathcal{B}_k$ \textup{\textbf{in sequence}}}{ 
%     $w' \leftarrow \wsupsubeq{}{}-\alpha  \gradsubeq{\wsupsubeq{t}{k}}f_{\mathcal{T}^{(s)}}(w^t_k)$ \tcp*[l]{Eq.\eqref{eq:adp}}
%     % \LeftComment{// \textit{optimize local objective} $\mathcal{L}_k$} \;
%     $w^k_t \leftarrow \underset{\wsupsubeq{}{}}{\text{argmin}} \ \mathcal{L}_k(w)$  \tcp*[l]{Eq.\eqref{eq:fed_maml}}
% }
% return $w$ to server \;
\caption{FedFSL-naive framework.} \label{algo:fed_maml}
%\end{algorithm2e}
\end{algorithm}
}

\begin{prop}
\label{prop:1}
If loss function $f_{\mathcal{T}}(w)$ in \eqref{eq:adp} satisfies
the strongly-convex conditions as in Corollary 1, Finn~\etal~\cite{finn2019online}
\footnote{ $f$ is \textit{G-Lipschitz}, $\beta$-\textit{smooth}, $\rho$-\textit{Lipschitz} Hessians and $\mu$-\textit{strongly} convex}
,  Algorithm~\ref{algo:fed_maml} converges at a rate of  $\mathcal{O}(\frac{1}{T})$ in which $T$ is the total number of every device's gradient updates during training.
\end{prop}
\begin{proof}
As $f_{\mathcal{T}}(w)$ is strongly-convex, 
Corollary 1, \cite{finn2019online} implies that the local FSL objective $\mathcal{L}_k$ in \eqref{eq:fed_def_adp_k} is also strongly-convex. By taking Algorithm~\ref{algo:fed_maml} as a FedAvg algorithm with a strongly-convex objective, Theorem 3, \cite{li2019convergence} implies that it converges at a rate of  $\mathcal{O}(\frac{1}{T})$ in which $T$ is total number of local gradient updates of all devices during training.
\end{proof}

Proposition~\ref{prop:1} shows that for a convex FSL objective,
e.g., the cross-entropy loss of a linear or logistic model with $L_2$-regularization, Algorithm~\ref{algo:fed_maml} converges.
For a broader family of non-convex models, such as deep neural networks, convergence of FedL is still an open research topic though 
% the local objective does not satisfy the strong convexity assumption. To provide extended theoretic convergence analysis to arbitrary local objective, 
some attempts have been made~\cite{smith2017cocoa, li2018federated}.
% assume that there exists some inexact solution which can be obtained with finite steps.
% that the dissimilarity between global and local optimal solution is bounded (Assumption 1,~\cite{li2018federated}), and the loss function with $L_2$ regularization is strongly convex, then a theoretic proof is available. 
%We will empirically study the convergence rate in \ref{sec:exp}.

We note that solving \eqref{eq:fed_maml} requires computing a gradient through $w'$, which is another function of gradient of $w$ as given by \eqref{eq:adp}, and thus requires computing the Hessian. Fortunately, fast Hessian-vector products~\cite{ghorbani2019investigation} are widely adopted to approximate the second-order information, which is equivalent of performing backward passes twice with SGD.
If we denote the number of parameters of the model as $W$, this results in $\mathcal{O}(W)$ computational time for one iteration of model update. Thus the total computational time of FedFSL-naive is $\mathcal{O}(W\cdot T)$. Note that Chen~\etal~\cite{chen2018fedmeta}  proposed similar federated meta-learning procedures for supervised learning tasks.

\section{Improving FedFSL with better coordination}

So far, we have provided FedFSL-naive as a straightforward way of performing distributed few-shot learning. However, one unresolved technical challenge is that meta-learning
depends on sampled episodes that contain only very few labeled data points. In data-scarce scenarios, even the data distribution over the clients could be the same, the high variance of the data 
may lead to quite distinct gradient descent directions, and thus
the trained few-shot models could become quite distinct over the clients. This results in model divergence in aggregation.
Similar observations were also found in FedL tasks with non-IID data \cite{zhao2018federated, li2018federated} but this problem has been amplified in the data-scarce scenarios we consider.

\begin{figure}
\begin{center}
\includegraphics[clip, trim=0 0 0 0, width=0.42\textwidth]{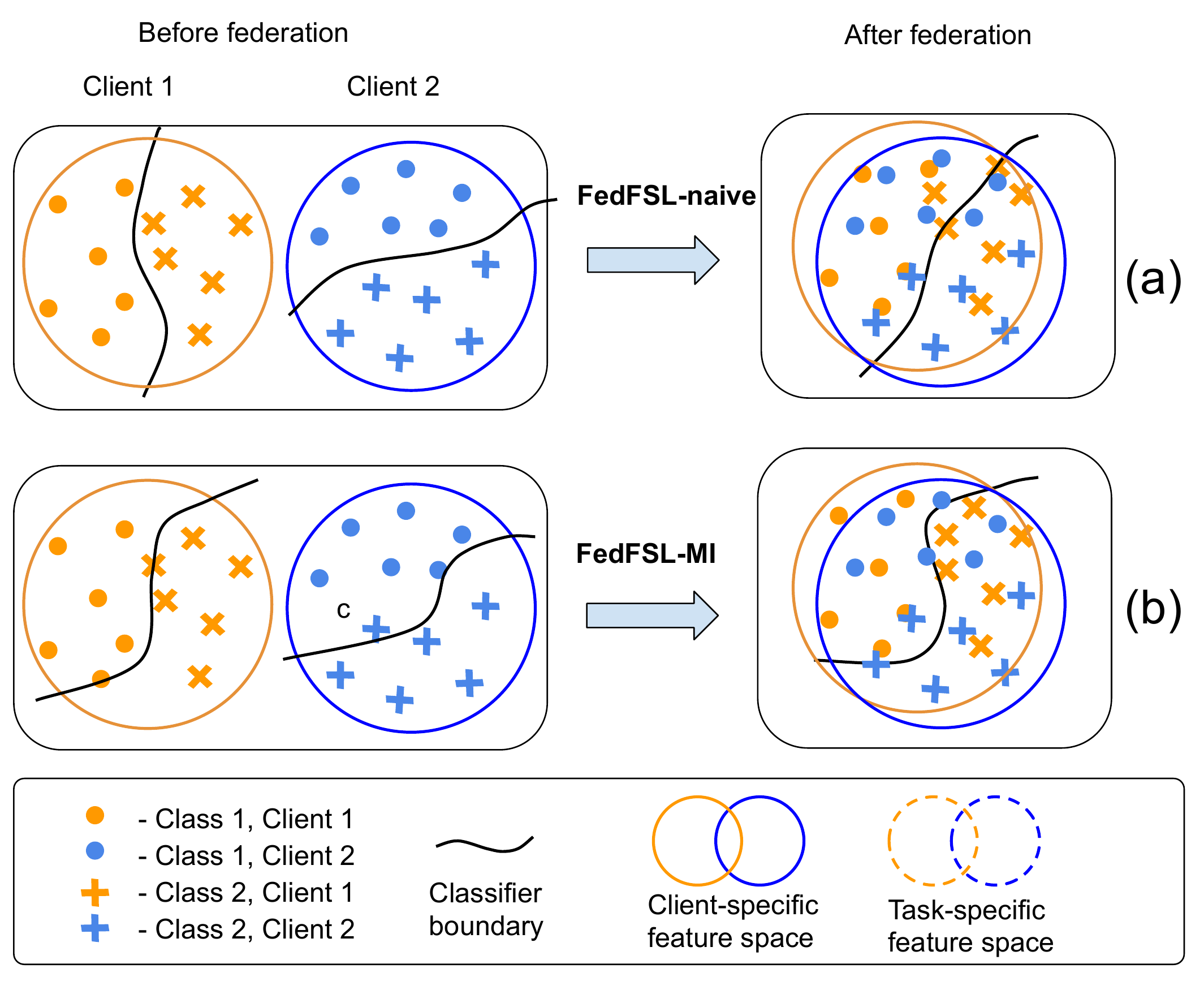}
\end{center}
\vspace{-5pt}
\caption{Illustration of decision boundaries learned by (a) FedFSL-naive and (b) FedFSL-MI in two-client case. }
\label{fig:demo_mcd}
\vspace{-10pt}
\end{figure}

In Fig.~\ref{fig:demo_mcd}(a), we illustrate a two-client case which follows FedFSL-naive scheme to learn models individually (left) first and average the models to obtain the federated decision boundary (right). However, the discrepancy between two client models makes them provide misaligned individual decision boundaries (left). Thus the aggregated central model provides less optimal federated decision boundary (right) with lots of misclassified data samples. 

In this section, we will discuss how to better coordinate client models with mutual information in \ref{sec:fed_maml_mi}, and we propose an adversarial learning procedures to further learn a discriminative feature space in \ref{sec:adv}.

\subsection{FedFSL with Mutual Information (MI)}
\label{sec:fed_maml_mi}
% We observe that recent improvement such as FedProx and FEDL propose to improve FedAvg by solving a surrogate function, with additional regularization to Recently, 

To better coordinate client models learned on distinct data sources,  we propose to regularize the local updates by minimizing the divergence between client models and the central model.
As the central model is shared with the clients at each round, it serves as an intermediate way of training the clients collaboratively without overfitting local data. 

\subsubsection{Mutual information (MI)}
Kullback-Leibler (KL) divergence is commonly used as a measure of the difference of two probability distributions. 
In a collection of multiple distributions, a summed pair-wise KL-divergence is used in
recent studies~\cite{belghazi2018mutual, zhang2018mutual} in ensemble learning to measure the total discrepancy of all those distributions, which we term it \textit{mutual information} (MI). 
In FedL, MI can be utilized to measure the discrepancy of all the participating client models. However, calculating the pairwise KL-divergence requires $\mathcal{O}(K^2)$ calculations which could impose heavy burdens to the central server.
We will propose a simplified MI proximal term and integrate it in FedFSL, and we target to minimize it for reducing internal discrepancy over client models. 

% Concretely, for client $k$ during $t$-th local update, let $p(w_k)$ and $p(w)$ be the probability output of client model $w_k$ and the global model $w^t$ respectively. 

\subsubsection{FedFSL-MI}
Formally, at $t$-th optimization round, we ask the central server to produce a k-exclusive global model $w^t_{\text{-}k}$ such that
\begin{equation}
\label{eq:fuse_mi}
w^{t}_{\text{-}k} =  \sum_{c=1,c\neq k}^K \frac{|\mathcal{B}_c|}{|\mathcal{B}_{\textit{-}k}|} \wsupsubeq{t}{c} \ ,
\end{equation} 
and send back to $k$-th client. The k-exclusive global model $w^{t}_{\text{-}k}$ is taken as an ensemble of all other client models except the $k$-th client. 

We now define the MI loss as the Kullback-Leibler (KL) divergence of probability outputs produced by the k-exclusive global model $w^{t}_{\text{-}k}$ and the client model $w_k$ over sampled tasks such that
\begin{equation}
\begin{split}
&\mathcal{L}_k^{MI}(w^t_{-k}, w_k) = \frac{1}{|\mathcal{B}_k|} \sum_{\Tcaleq_k} D_{KL} \left (p(w^t_{-k})~ ||~ p(w_k) \right ) \\
=&\frac{1}{|\mathcal{B}_k|} \sum_{\Tcaleq_k} (p(w^t_{-k}) \cdot \log p(w^t_{-k}) - p(w^t_{-k}) \cdot \log p(w_k)), \\
\end{split}
\label{eq:loss_kd}
\end{equation}
in which $p(\cdot)$ is the probability outputs of an FSL model. Given an $N$-way FSL task $\mathcal{T}_k$, $p(w)$ is the normalized $N$-way predictions over $N$ classes that sums to one. We aim to minimize MI in order to reduce the discrepancy.

By integrating MI into the original local FedFSL objective function \eqref{eq:fed_def_adp_k},
our new target is to jointly minimize the MI loss together with the local FSL task loss such that
% The client is updating its local model with a combined target of few-shot learning task objective \eqref{eq:fed_maml} and a weighted MI regularization target  \ref{eq:fuse_mi} such that
% By combining the FedFSL task in Eq~(\ref{eq:fed_maml}) and mutual information as regularization term, we obtain the objective
\begin{equation}
\begin{split}
&\wsupsubeq{*,t}{k}=\underset{w_k}{\text{min}} \ \mathcal{L}_k(w_k) + \gamma \mathcal{L}_k^{MI}(w^{t}_{\text{-}k}, w_k),
\end{split}
\label{eq:fed_maml_mi}
\end{equation}
in which the weight $\gamma>0$ can be searched by cross-validation. We call this new method \textbf{FedFSL-MI} (FedFSL with Mutual Information regularization).

As the k-exclusive global model is different from client to client, the central server needs to compute the global model $K$ times.
In practice, when $K$ is large (e.g., $\ge$10), we can conveniently reuse the aggregated central model $w^{t}$ as in \eqref{eq:fuse} to
approximate $w^{t}_{\text{-}k}$, leading to no additional computation cost. We will default to use $w^t$ to approximate $w^{t}_{\text{-}k}$ in our experiments to reduce computations in simulated mobile devices, and we will compare them in Sec.\ref{sec:ab_mi}.
% In two-client case, $w^{t}_{\text{-}k}$ can be obtained by directly exchanging models with each other.
% \hl{We summarize the procedures} in Algorithm~\ref{algo:fed_maml} with local objective \eqref{eq:fed_maml_mi}.

We illustrate the intuition of using MI in Figure~\ref{fig:demo_mcd}(b).  As we minimize the discrepancy among client models, we encourage the decision boundaries to be consistent across the clients. Thus the federated model could produce a better aligned decision boundary. In the empirical study, we will show that it leads to a significant improvement over FedFSL-naive. 
However, the decision boundaries could become complex due to the consistency constraint (Figure~\ref{fig:demo_mcd}(b) right), which we will discuss in next section.

\subsection{Improving feature space with adversarial learning}
\label{sec:adv}

One technical disadvantage of FedFSL-MI is that constraining the decision boundaries to be similar over clients would develop a complex classifier that overfits to training tasks.
However, the classes of testing data for FSL are different from the base classes of training data, which makes the complex decision boundary not useful to unseen tasks. This also presents a key difference between FSL and conventional supervised learning.

\subsubsection{Feature space}
We aim to improve the FedFSL-MI by learning a central model that can produce a better-aligned decision boundary for unseen tasks.
Recent studies of metric learning~\cite{schroff2015facenet, Gidaris_2018_CVPR} have shown that learning a good feature space is beneficial to various tasks as it provides good 
representations (also known as features or feature embeddings) for data samples. In an ideal feature space, samples of the same class or similar classes are close to each other, while samples of different classes are far away. For example, images of cats and tigers are close in feature space, while tigers and wolves could be far away due to their distinct visual features. 
% Similarly, words of close semantic meanings are usually closer in word embedding space~\cite{pennington2014glove}.

Researchers have also found that a representative feature space is a kind of transferable knowledge that can be used to learn unseen data samples. For example, pre-trained vision recognition models (i.e., ResNet)~\cite{sharif2014cnn, ren2015faster} and language models (i.e. BERT)~\cite{pennington2014glove, vaswani2017attention} can produce off-the-shelf image/language representations for various tasks. Few-shot classification, as we consider in this paper, will especially benefit from a discriminative and transferable feature space if such a feature space can be derived properly and efficiently in distributed scenarios. We will show it is feasible in next sections.
% as the decision boundary will be much easier to be learned with well represented data.

\subsubsection{Learn a consistent feature space}
To our best knowledge, how to learn a consistent and discriminative feature space with FedL has never been studied before.
The difficulty is how to construct a consistent feature space over many clients without sharing data.
% We will explore an efficient way of constructing such a feature space by explicitly optimizing for a discriminative feature generator for FedFSL models. 
Motivated by recent progress in Generative Adversarial Networks (GANs)~\cite{goodfellow2014generative, saito2018mcd}, we will decompose an FSL model as a feature generator and a classifier (i.e., discriminator) which can be optimized in an alternative and iterative fashion.
This new adversarial learning approach is named as \textbf{FedFSL-MI-Adv} (FedFSL with Mutual Information regularization and Adversarial learning).

% Notations of FedFSL-MI-Adv in details.
We first introduce some new notations to facilitate discussion.
Without loss of generality, a few-shot classification model can be represented as a feature generator $\Theta$ and a classifier $\theta$. For a given data sample $x$, we denote its generated feature as $f_{\Theta}(x)$. The output logits of the classification model is derived by applying the classifier on the feature such that $f_{\theta}\circ f_{\Theta}(x)$. Thus the predicted  probabilistic distribution over $N$ classes is denoted as
$p(\Theta, \theta)=\sigma (f_{\theta}\circ f_{\Theta}(x))$ in which $\sigma$ is the softmax function.  

In centralized training, the
feature generator and the classifier could be learned with supervised learning without many tricks. However, in distributed scenarios, we have to additionally consider aligning feature space learned with many clients. We propose a novel procedure that
alternatively trains the classifier and the generator as two opponents. We train the client model classifier to maximize the difference between its predictions and central model predictions, while train the client feature generator to minimize the difference. We will explain the details and intuitions.

\begin{figure}
\begin{center}
\includegraphics[clip, trim=0 0 0 0, width=0.42\textwidth]{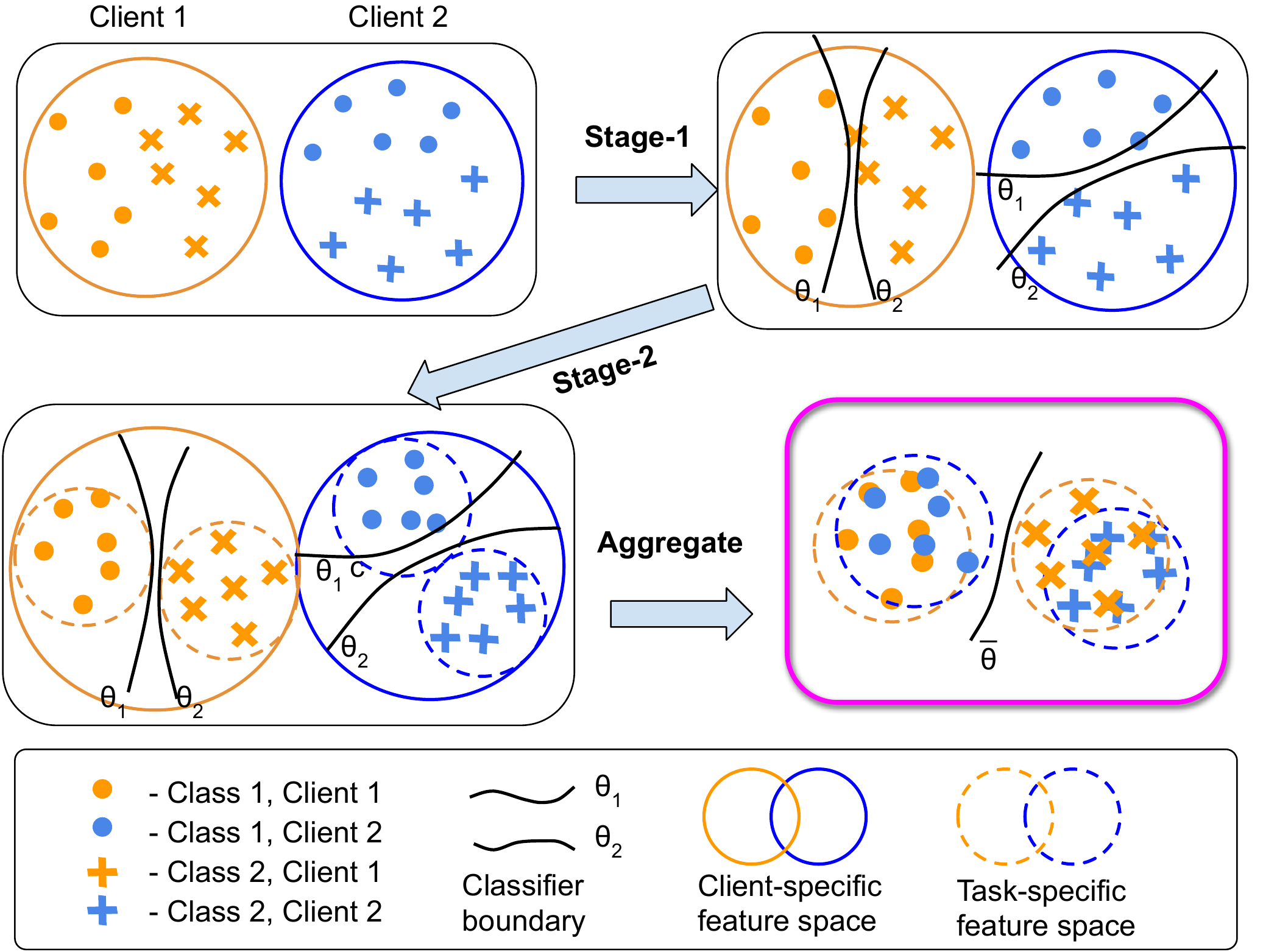}
\end{center}
\vspace{-5pt}
\caption{An example of federated decision boundaries learned by FedFSL-MI-adv with two-stage adversarial learning on two clients of different data distributions. }
\label{fig:demo_mcd_adv}
\vspace{-15pt}
\end{figure}

% We now propose an adversarial learning based approach to improve FedFSI-MI by explicitly optimizing task-specific decision boundaries across the clients. 
% In comparison, FedFSL-MI optimizes the client model as a whole to minimize the difference between client and central model, thus may lead to a over-trained classifier. However, a task-specific classifier is usually not useful to new tasks on novel classes.

\subsubsection{Adversarial learning procedure}
We design a two-stage adversarial learning procedure for the local update for explicitly learning a consistent feature space.
In overall, the $(t\text{-}1)$-th communication round ends up by aggregating the
% includes updating $k$-th client model which yields $w^t_k=\Ttsupsubeq{t}{k}$, and aggregating 
client models to a central model in \eqref{eq:fuse} and sending it back to clients as $w_k=\Ttsupsubeq{}{k}$. 
At the beginning of next round of local update, 
each client initializes a new classifier $\theta_k'$.
The feature generator $\Theta_k$ and two classifiers $\theta_k$ and $\theta_k'$ are all trainable and will involve in a two-stage adversarial training procedure as follows.

% We design a two-stage training process of our proposed method to alternatively update the classifiers $\theta_k$ and the generator $\Theta_k$ in adversarial fashion to build a discriminative feature generator.

\begin{itemize}[leftmargin=*]
    \item \textbf{Training stage-1} is to 
train two classifiers to produce \emph{distinct} decision boundaries, in the motivation of detecting ambiguous data samples in current feature space. Ambiguous samples are those lying near the decision boundaries that tend to be misclassified by two different classifiers, as shown in Fig.~\ref{fig:demo_mcd_adv}(stage-1). Intuitively, detecting those samples is the prerequisite of optimizing a feature space that resolves the ambiguity.
During this stage, the feature generator $\Theta_k$ remains fixed, while $\theta_k$ and $\theta_k'$ are updated.

We first define the adversarial loss to measure the difference of two classifiers $\theta_k$ and $\theta_k'$ by the KL divergence of their probabilistic outputs $p(\Theta, \theta)$ such that
\begin{equation}
\begin{split}
&\mathcal{L}_k^{adv}(\theta_k, \theta_k', \Theta_k) 
% = \sum_{\Tcaleq_k} \mathcal{L}^{kl}_{\Tcaleq} (f_{[\theta, \Theta]}, f_{[\theta_k, \Theta]}) \\
= \frac{1}{|\mathcal{B}_k|}\sum_{\Tcaleq_k} D_{KL} \left (p(\Theta_k, \theta_k)~ ||~ p(\Theta_k, \theta_k') \right ).
% =&\sum_{\Tcaleq_k} \mathrm{KL}\left( \sigma \left (\frac{f_{\theta_k^c} \circ f_{\Theta_k}}{T} \right),\sigma \left (\frac{f_{\theta_k} \circ f_{\Theta_k}}{T} \right ) \right ) \\
\end{split}
\label{eq:loss_adv}
\end{equation}
We simultaneously minimize the FSL local objective \eqref{eq:fed_def_adp_k} while \emph{maximize} the adversarial loss to encourage the disagreement of the two classifiers. Formally,
we define the objective of stage-1 as a combination of local task objective $\mathcal{L}_k$ \eqref{eq:fed_def_adp_k} and adversarial loss $\mathcal{L}_k^{adv}$ \eqref{eq:loss_adv} with weight $\eta>0$ such that
% \begin{equation}
% \begin{gathered}
% \underset{\theta_k,\theta_k'}{\text{min}}  \ \mathcal{L}_k^{st\text{-}1}(\theta_k, \theta_k'; \Theta_k) \\ 
% = \mathcal{L}_k (\theta_k; \Theta_k) +  \mathcal{L}_k (\theta_k'; \Theta_k) - \eta \mathcal{L}_k^{adv}(\theta_k, \theta_k' ; \Theta_k)
% \end{gathered}
% \label{eq:fed_mi_adv_1}
% \end{equation}
\begin{multline}
\underset{\theta_k,\theta_k'}{\text{min}}  \ \mathcal{L}_k^{st\text{-}1}(\theta_k, \theta_k'; \Theta_k) \\ 
= \mathcal{L}_k (\theta_k; \Theta_k) +  \mathcal{L}_k (\theta_k'; \Theta_k) - \eta \mathcal{L}_k^{adv}(\theta_k, \theta_k' ; \Theta_k).
\label{eq:fed_mi_adv_1}
\end{multline}

\item \textbf{Training stage-2} is to minimize adversarial loss for learning the discriminative feature generator.
In this stage, we fix the classifiers $\theta_k$ and $\theta_k'$ and train the generator $\Theta_k$ to minimize the discrepancy of the two classifiers measured by the adversarial loss.
The intuition is shown in Fig.~\ref{fig:demo_mcd_adv}(stage-2):
by \textit{minimizing} \eqref{eq:loss_adv},  the feature generator $\Theta$ is learning to push ambiguous data samples away from the decision boundaries, so that both classifiers could make the right predictions and their discrepancy gets reduced. As a result, the feature space (dashed circles) generated by $\Theta$ is trained to be discriminative which produces larger inter-class margins.

Formally, we define the objective of stage-2 as a combination of local task objective $\mathcal{L}_k$ and adversarial loss $\mathcal{L}_k^{adv}$ with weight $\lambda>0$ such that
\begin{multline}
\underset{\Theta_k}{\text{min}}  \ \mathcal{L}_k^{st\text{-}2}(\Theta_k; \theta_k, \theta_k') \\
= \mathcal{L}_k (\Theta_k; \theta_k) +  \mathcal{L}_k (\Theta_k; \theta_k') + \lambda \mathcal{L}_k^{adv}(\Theta_k ; \theta_k, \theta_k').
\label{eq:fed_mi_adv_2}
\end{multline}
\end{itemize}

%In practice we can also add entropy loss on $\theta$ to control its complexity.

% , as shown in top row of Figure~\ref{fig:two_stage}.
% to become as distinct as possible regarding their 
% Maximize discrepancy between global classifier and local classifier, while we also minimize cross-entropy loss as well as classifier entropy loss. In this stage, we fix feature generator $\Theta_k$ and update local classifier $\theta_k$, as shown in top of Figure~\ref{fig:pipe_mcd}.
% This time, we maximize the discrepancy to encourage classifiers to learn diverse task boundaries.

% in which $\lambda>0$ is the weight of the discrepancy loss which can be searched by cross-validation. 
By training the classifiers and the feature generator in an adversarial manner, we iteratively optimize the model to learn a discriminative feature generator which helps boost few-shot learning on unseen tasks.

In our toy example, the feature generator learned by cat and dog images is likely to distinguish these two categories by eyes and ears and other unique features of these two species. Thus, cat-like and dog-like images are projected with large margins in the feature space learned explicitly by our approach. This space is transferable to tasks such as classifying tiger and wolf images. In reality, the richer and more representative are the base classes, the more discriminative the feature space will be. 
Recently, centralized machine learning models~\cite{sharif2014cnn, vaswani2017attention} have been shown to be capable of learning generic and versatile feature spaces on complex structured data such as images and texts. We have shown that such a feature space can also be efficiently learned in distributed scenarios.

In conclusion, we have proposed a novel way of learning a discriminative feature space in FedFSL with an adversarial learning strategy.
We summarize FedFSL-MI-Adv in Algorithm~\ref{algo:fed_mi_adv}.
{\SetAlgoNoLine
\begin{algorithm}[htp]
\small
\DontPrintSemicolon
\LinesNumberedHidden
\KwIn{A set of $K$ federated clients. A local FSL objective $\mathcal{L}_k$ for each client $k$.
}
\KwOut{A global model $w=\Ttsupsubeq{}{}$ optimized for FSL task.}
\textbf{Server executes:} \;
\Indp Initialize global model $w^0=\Ttsupsubeq{0}{}$ \;
$t \leftarrow 1$ \;
\While{t $\leq$ maximum rounds $T$}{
    % $m \leftarrow \max(C\cdot K, 1)$ \;
    % $S_t \leftarrow $ (a random subset of $m$ clients) \;
    \For{\textup{each client} $k$ \textup{\textbf{in parallel}}}{ 
    $\Ttsupsubeq{t}{k} \leftarrow $ ClientUpdate$(\Ttsupsubeq{t}{})$
    }
    Clients send models $\Ttsupsubeq{t}{1...K}$ to server \;
    $\Ttsupsubeq{t+1}{} \leftarrow \sum_{k=1}^K \frac{|\mathcal{B}_k|}{|\mathcal{B}|} \Ttsupsubeq{t}{k}$ \;
    The server sends $\Ttsupsubeq{t+1}{}$ back to clients \;
    $t \leftarrow t+1$ \;
 }
Return $\Ttsupsubeq{t}{}$ \;
\;
\Indm \textbf{ClientUpdate}$(\Ttsupsubeq{}{})$: \;
% \Indp $\mathcal{B}_k \leftarrow$ (split local data into batches of episodes) \;

\Indp \KwIn{global model from previous round $\Ttsupsubeq{t}{}$}
\KwOut{updated local model $\Ttsupsubeq{t}{k}$}

Sample a batch of episodes $\mathcal{B}_k=\{\mathcal{T}_1,...,\mathcal{T}_n\}$ \;
$\Ttsupsubeq{}{k} \leftarrow \Ttsupsubeq{t}{}$ \;
Initialize a new classifier $\theta_k'$ \;
$\theta_k, \theta_k' \leftarrow $ Solve Eq.\eqref{eq:fed_mi_adv_1} with SGD \;
$\Theta_k \leftarrow $ Solve Eq.\eqref{eq:fed_mi_adv_2} with SGD \;
$\Ttsupsubeq{t}{k} \leftarrow \Ttsupsubeq{}{k}$ \;
Return $\Ttsupsubeq{t}{k}$ \;

% $\theta_k^c \leftarrow \theta_k^t$\;
% \For{\textup{each episode} $\mathcal{T} \in \mathcal{B}_k$}{
%     // \textit{adapt to new task as in \eqref{eq:adp}} \;
%     $\Ttsupsubeq{}{k} \leftarrow \Ttsupsubeq{}{k}-\alpha \gradsubeq{\Ttsupsubeq{}{k}}f_{\mathcal{T}^{(s)}}(\Ttsupsubeq{}{k})$  \;
%     $[\Theta_k, \theta_k'] \leftarrow [\Theta_k, \theta_k']-\alpha \gradsubeq{[\Theta_k, \theta_k']}f_{\mathcal{T}^{(s)}}([\Theta_k, \theta_k'])$ \;
%     // \textit{stage-1 optimization as in \eqref{eq:fed_mi_adv_1}} \;
%     $\theta_k, \theta_k' \leftarrow \underset{\theta_k,\theta_k'}{\text{min}}  \ \mathcal{L}_k^{st\text{-}1}(\theta_k, \theta_k' ; \Theta_k)$ \;
%     // \textit{stage-2 optimization as in \eqref{eq:fed_mi_adv_2}} \;
%     $\Theta_k \leftarrow \underset{\Theta_k}{\text{min}} \ \mathcal{L}_k^{st\text{-}2}(\Theta_k; \theta_k, \theta_k')$
% }
\caption{FedFSL-MI-Adv algorithm.} \label{algo:fed_mi_adv}
%\end{algorithm2e}
\end{algorithm}
}

% {\SetAlgoNoLine
% \begin{algorithm}[htpb]
% \DontPrintSemicolon
% \LinesNumberedHidden
% \KwIn{A set of $K$ federated clients. A local FSL objective $\mathcal{L}_k$ for each client $k$.
% }
% \KwOut{A global model $\wsupsubeq{}{}$ optimized for FSL task.}
% \textbf{Server executes:} \;
% \Indp Initialize global model $w^0$ \;
% $t \leftarrow 1$ \;
% % Randomly initialize $\wsupsubeq{}{k}$ for all clients. \;
% %\For{\textup{each round} $t = 1,2,\dots,T$ } {
% \While{not done}{
%     \For{\textup{each client} $k$ \textup{\textbf{in parallel}}}{ 
%     $w_k^{t} \leftarrow $ ClientUpdate{$(k,w)$} 
%     }
%     Collect $w_k^{t}$ from all clients \;
%     $w \leftarrow \sum_{k=1}^K \frac{|\mathcal{B}_k|}{|\mathcal{B}|} w^{t}_k$  \tcp*[l]{model avg}
%     Send $w$ back to clients \;
%     $t \leftarrow t+1$
%  }
% \;
% \Indm \textbf{ClientUpdate}$(k,w)$: \;
% \Indp $\mathcal{B}_k \leftarrow$ (split local data into batches of episodes) \;
% \For{\textup{each episode} $\mathcal{T}_k \in \mathcal{B}_k$}{
%     $\wsupsubeq{}{} \leftarrow \wsupsubeq{}{}-\eta  \gradsubeq{\wsupsubeq{}{}}f_{\mathcal{T}^{(s)}_k}(w)$ \tcp*[l]{Eq.\eqref{eq:adp}}
%     % \LeftComment{// \textit{optimize local objective} $\mathcal{L}_k$} \;
%     $w \leftarrow \underset{\wsupsubeq{}{}}{\text{argmin}} \ \mathcal{L}_k(w)$  \tcp*[l]{Eq.\eqref{eq:fed_def_adp_k}}
% }
% % return $w$ to server \;
% \caption{FedFSL-naive framework.} \label{algo:fed_maml}
% %\end{algorithm2e}
% \end{algorithm}
% }

\section{Experiments and Discussions}
\label{sec:exp}
We first provide details of the model architecture, parameter settings, and datasets that we use in the experiments.
Then we visualize the decision boundaries of our approaches with a toy example. We then demonstrate the performance of our proposed algorithms with two typical few-shot classification tasks -- 5-way 1-shot and 5-way 5-shot -- on three benchmark datasets which cover machine vision and NLP tasks. We will make in-depth discussions.

\subsection{Model architecture} 
We utilize a 12-layer deep neural network (DNN) as our base model (adopted from ResNet-12~\cite{he2016deep} which is commonly used for image classification tasks~\cite{sun2019mtl, oreshkin2018tadam, zhang2018metagan}).
It consists of the feature extractor with 12 convolutional layers and about 5 million parameters in total, as well as two fully-connected layers with a ReLU nonlinearity as the classifier. We update model parameters by Adam solver~\cite{kingma2014adam} with a fixed learning rate $10^{-3}$. We
set the adaptation step size $\alpha=0.01$ in \eqref{eq:adp}, the mutual information weight $\gamma=0.2$, and the stage-1/-2 discrepancy loss weight $\eta=\lambda=0.1$ in \eqref{eq:fed_mi_adv_1} and \eqref{eq:fed_mi_adv_2}.
%We sample 200 episodes for each round of local update.

% We utilize a small 4-layer deep neural network (DNN) as our base model which is commonly used for image classification tasks~\cite{sun2019mtl, oreshkin2018tadam, zhang2018metagan, finn2017model}.
% It consists of the feature extractor $\Theta$ with 4 convolutional layers and about 0.3 million parameters in total, and two fully-connected layers with a ReLU nonlinearity as the classifier $\theta$.

%\subsection{Parameter settings} 
%We simulate the FedL environment and implement DNNs with PyTorch~\cite{paszke2017automatic}.

\subsection{Datasets} 
We briefly describe three benchmark datasets that are commonly used in studying FSL and FedL~\cite{finn2017model,munkhdalai2017meta,li2017meta,sun2019mtl,chen2018fedmeta}.
\begin{itemize}[leftmargin=*]
\item \textbf{\mini}~\cite{vinyals2016matching} is based on a small portion of the full ImageNet images~\cite{deng2009imagenet}.
It has 100 classes of images split to 64/16/20 as train/val/test sets. Each class has 600 images with a resolution of 84 $\times$ 84. 
\item \textbf{FC100}~\cite{oreshkin2018tadam} is based on CIFAR-100 images that has 100 classes split to 60/20/20 as train/val/test splits. Each class has 600 images with a low resolution 32$\times$32. It is more challenging because of the low image quality.
\item \textbf{Sent140}~\cite{caldas2018leaf} is a benchmark federated learning  dataset for 2-way sentiment classification (positive and negative). We sampled from this dataset 10,000 annotated tweets provided by 310 twitter users and split them to train/val/test sets with provided tools. Each tweet has 1-20 English words. We tokenize the sentences and keep only common words which have GloVe representations~\cite{pennington2014glove}. 
\end{itemize}

\subsection{MNIST Example}

We provide a simple example on MNIST digit dataset in Fig.~(\ref{fig:mnist_example}), to visualize and compare the decision boundaries of FedFSL-MI and FedFSL-MI-Adv.
We consider the 5-way 1-shot FSL task here: train a digit classification model on data of digits 0-4, and test its few-shot classification capability on digits 5-9 by observing just one labeled sample per class.
To better visualize the results, we manipulate the feature generator to produce a 2-dim feature for each input digit sample. 
% During testing, we sample 100 batches of digital data of classes of 5-9, each with 1 image for adaptation and 10 images for query per category.

In Fig.~\ref{fig:mnist_example}, 
we plot the testing data samples from digit class 5-9 by projecting their features produced by the feature generators. We also depict the decision boundaries of the classifiers. Data samples of different classes are with different colors. 
We observe that FedFSL-MI-Adv (left) produces more distinguishable decision boundaries than FedFSL-MI (right) as expected. The two algorithms achieve a few-shot classification accuracy of 87.5\% and 83.6\%, respectively.
The least accurate class recognized by FedFSL-MI-Adv is digit '9' (purple) of 73\% correctness rate, with 18\% misclassified as '6' (orange); while for FedFSL-MI is digit '6' (orange) of 64\% correctness rate, with 23\% misclassified as '8' (red). 
This indicates that our adversarial learning approach proposed in Section~\ref{sec:adv} boosts the FedFSL task by constructing a more discriminative and transferable feature space for FSL.

\begin{figure}[t]
\begin{center}
\includegraphics[clip, trim=15 0 0 0, width=0.40\textwidth]{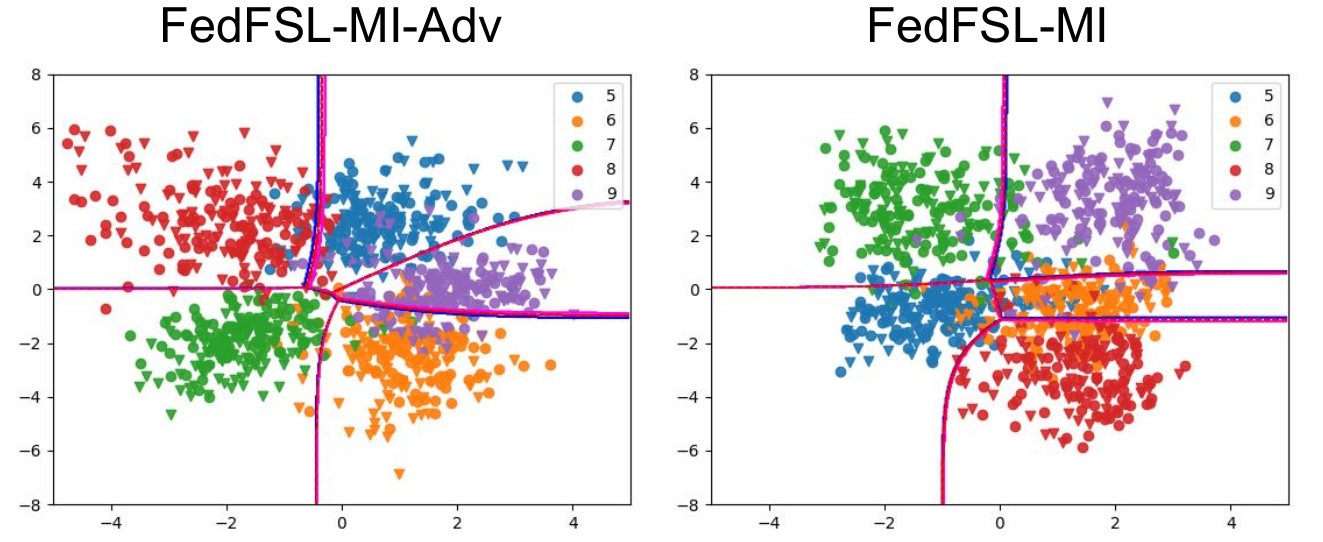}
\end{center}
\vspace{-10pt}
 \caption{Visualization of decision boundaries of FedFSL-KD and FedFSL-KD-Adv at different epochs.}
\label{fig:mnist_example}
\vspace{-10pt}
\end{figure}

\subsection{Results on benchmark datasets}
We experiment with our proposed three methods (FedFSL-naive, -MI, -MI-Adv) and two additional baselines (FSL-local and FedFSL-prox) for comparison.
\begin{itemize} [leftmargin=*]
\item FSL-local is a non-distributed baseline of training an individual FSL model for each client on local data and averaging their results on the shared testing tasks. 
\item FedFSL-prox is a variant of FedFSL-naive by adding a weight regularization term as FedProx~\cite{li2018federated} in objective.
\end{itemize}
We partition the data samples in IID and non-IID ways. For IID partition, data samples of each class are uniformly distributed to each client. To perform non-IID partition, we follow~\cite{hsu2019measuring,yu2020salvaging} by dividing data samples to all clients class-by-class with Dirichlet distribution of concentration parameter $\alpha = 1.0$. In Fig.\ref{fig:data_nid}, we show an example of such a partition of 64 training classes of miniImageNet on a randomly chosen client, when total device number is 2 to 30.

%\begin{figure}[hbt]
%\begin{center}
%\includegraphics[clip, trim=10 0 10 10, width=0.25\textwidth]{figures/data_dist_device_1.png_train.png}
%\end{center}
%\vspace{-4pt}
%\caption{An example of Non-IID data allocated to 2 to 20 federated devices with Dirichlet distribution $\alpha=1.0$.}
%\vspace{-8pt}
%\label{fig:data_nid}
%\end{figure}

\makeatletter\def\@captype{figure}\makeatother
\hspace{-15pt}
\begin{minipage}{0.22\textwidth}
\centering
\includegraphics[scale=0.28, trim=10 20 0 0]{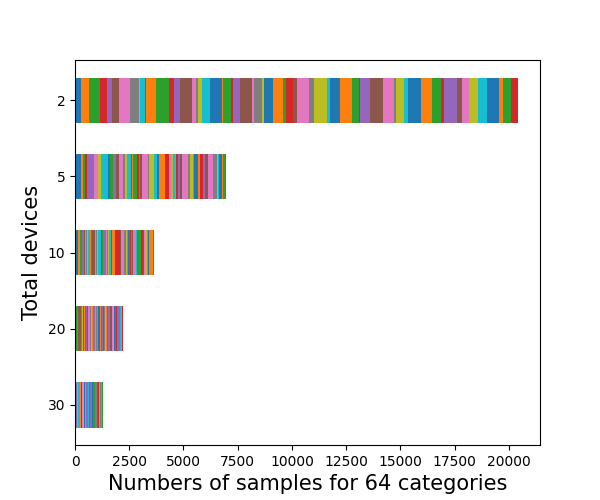}
\vspace{-13pt}
\caption{Non-IID data.}
\label{fig:data_nid}
\end{minipage}
\makeatletter\def\@captype{table}\makeatother
\hspace{-5pt}
\begin{minipage}{.26\textwidth}
\centering
\small
\vspace{10pt}
\begin{tabular}{lcc} \toprule  
\multirow{2}{*}{\textbf{Method}} & \multicolumn{2}{c}{\textbf{Non-IID}} \\ 
\cmidrule{2-3}     & 1-shot  & 5-shot \\ \toprule
FedFSL-local & 59.70\% & 66.68\% 	\\
FedFSL-naive &  68.85\% & 70.62\%	\\
FedFSL-prox &  70.77\% & 72.25\%	\\
FedFSL-MI  &  70.37\% &  73.25\% \\
FedFSL-MI-Adv  & \textbf{71.35\%} & \textbf{76.00\%}   \\  \midrule \bottomrule
\end{tabular}
\vspace{-2pt}
\caption{Sent140 results.}
\label{tab:res_sent140}
\end{minipage}

\begin{table}[htb]
\caption{Results on benchmark datasets.}
\small
\begin{subtable}{\linewidth}
{
\begin{center}
\begin{tabular}{lcccc} \toprule  
\multirow{2}{*}{\textbf{Method}} & \multicolumn{2}{c}{\textbf{IID}} & \multicolumn{2}{c}{\textbf{Non-IID}} \\ 
\cmidrule{2-3} \cmidrule{4-5}    & 1-shot   & 5-shot  & 1-shot   & 5-shot \\ \toprule
FSL-local  & 50.83\%	&  67.47\% & 48.08\% & 63.25\%	 \\
FedFSL-naive~\cite{chen2018fedmeta} & 53.00\%	 &  67.63\%	& 49.95\% & 66.11\% \\
FedFSL-prox~\cite{li2018federated} &  53.03\%	 &  69.05\%	& 50.08\% & 68.53\% \\ \midrule
FedFSL-MI (ours)  & 54.98\% & 69.07\% & 51.07\% & 68.57\% \\
FedFSL-MI-Adv (ours)  & \textbf{56.42\%}	&	\textbf{70.92\%}  & \textbf{53.69\%}& \textbf{69.61\%}    \\  \midrule \bottomrule
\end{tabular}
\end{center}
}
\vspace{-2pt}
\caption{MiniImageNet results.}
\end{subtable}

\begin{subtable}{\linewidth}
{\begin{center}
\begin{tabular}{lcccc} \toprule  
\multirow{2}{*}{\textbf{Method}} & \multicolumn{2}{c}{\textbf{IID}} & \multicolumn{2}{c}{\textbf{Non-IID}} \\ 
\cmidrule{2-3} \cmidrule{4-5}    & 1-shot   & 5-shot  & 1-shot   & 5-shot   \\ \toprule
FSL-local &  36.45\% 	&  47.68\% & 35.90\% & 52.93\%	   \\
FedFSL-naive~\cite{chen2018fedmeta} & 38.42\%	 & 49.97\%	 & 36.11\% & 51.58\%	\\
FedFSL-prox~\cite{li2018federated} & 37.62\% & 48.99\%	& 36.95\% & 53.02\% \\ \midrule
FedFSL-MI (ours)  & 39.78\%	& 50.65\% & 38.08\% & 52.98\% \\
FedFSL-MI-Adv (ours)  & \textbf{40.22\%}	& \textbf{51.18\%}	 & \textbf{38.51\%} & \textbf{54.43\%}       \\ \midrule\bottomrule
\end{tabular}
\end{center}}
\vspace{-2pt}
\caption{FC100 results.}
\vspace{-10pt}
\end{subtable}

\label{tab:res_iid}
\end{table}

\subsubsection{Results on Image Classification}
In Table~\ref{tab:res_iid}, we compare our methods on miniImageNet (a) and FC100 (b) with 1-shot / 5-shot tasks learned by a federation of 10 clients. 
We observe that
\begin{itemize}[leftmargin=*]
\item \emph{FedFSL-MI-Adv outperforms others}. For both IID and non-IID case, FedFSL-MI-Adv consistently outperforms others on both 1-shot and 5-shot tasks for both datasets. 
\item For 1-shot and 5-shot IID task on miniImageNet, FedFSL-MI-Adv achieves the best accuracy of 56.42\% and 70.92\%, which outperforms the second best FedFSL-MI by more than 2.6\% and 2.7\% respectively and relatively. A similar trend has also been observed for FC100.
%: for 1-shot and 5-shot tasks,  FedFSL-MI-Adv achieves the best accuracy of 38.01\% and 48.26\%, outperforming the second best FedFSL-MI by more than 7.8\% and 5.1\%, respectively.
\item For non-IID task, FedFSL-MI-Adv outperforms the second best FedFSL-MI by more than 5\% and 3.1\% in 1-shot case on miniImageNet and FC100 datasets respectively, and outperforms FedFSL-naive by more than 7.2\% and 5.5\% respectively. This indicates that our designed modules indeed help achieve a better federated model especially for non-IID case.
%\item Using distributed data always outperforms the non-distributed method FSL-local. The FSL capacity gains further by enforcing model consistency such as FedFSL-MI, and adding feature learning such as FedFSL-MI-Adv.
\item For FC100 5-shot task,  FedFSL-MI-Adv performs better on non-IID partitions than IID partitions, with an accuracy 54.43\% (non-IID) v.s. 51.18\% (IID) shown by last line in Table~\ref{tab:res_iid}(b). One explanation is that non-IID partitions force each client model to learn on distinct local tasks where certain data classes get sampled more times and thus get represented well. As FedFSL-MI-Adv further aligns the feature spaces of all clients, we could derive a more representative joint feature space with the global model.
% \item FedFSL-MI-Adv consistently outperforms FedFSL-MI, by margins of more than 5.8\% and 5\% for 1-shot and 5-shot tasks on miniImageNet and even more on FC100. This shows again that the adversarial learning strategy boosts the FedFSL tasks significantly.
\item Using FedProx~\cite{li2018federated} performs no better than our FedFSL-MI. This is because FedProx directly constrains client model weights to be closer to global model, while our FedFSL-MI softly optimizes the model outputs of them to be closer, which makes the training end-to-end and easier to optimize.

% FSL-local performed  non-distributed and does not jointly train a central model with all data sources. FedFSL-MI-Adv surpasses FSL-local by 12.5\% (46.82\% v.s. 41.64\%) relatively for 1-shot task, and 16.6\% (61.17\% v.s. 52.4\%) relatively for 5-shot task on miniImageNet, and that relative advantages are 15\% and 21.9\% on FC100. This result reinforces our motivation that the study of distributed FSL is imperative and beneficial especially for mobile computing scenarios with scarce and distributed data sources.
% \item FedFSL-MI and FedFSL-prox achieved comparable results for both 1-shot and 5-shot tasks on both benchmark datasets. This could be explained by that both approaches have proximal term to regularize  client models with their probabilistic outputs \eqref{eq:fed_maml_mi} or model weights to be closer to the global model.
\end{itemize}

\begin{figure}[t]
\centering
\begin{subfigure}{.46\textwidth}
\includegraphics[clip, trim=0 0 20 0, width=1.0\textwidth]{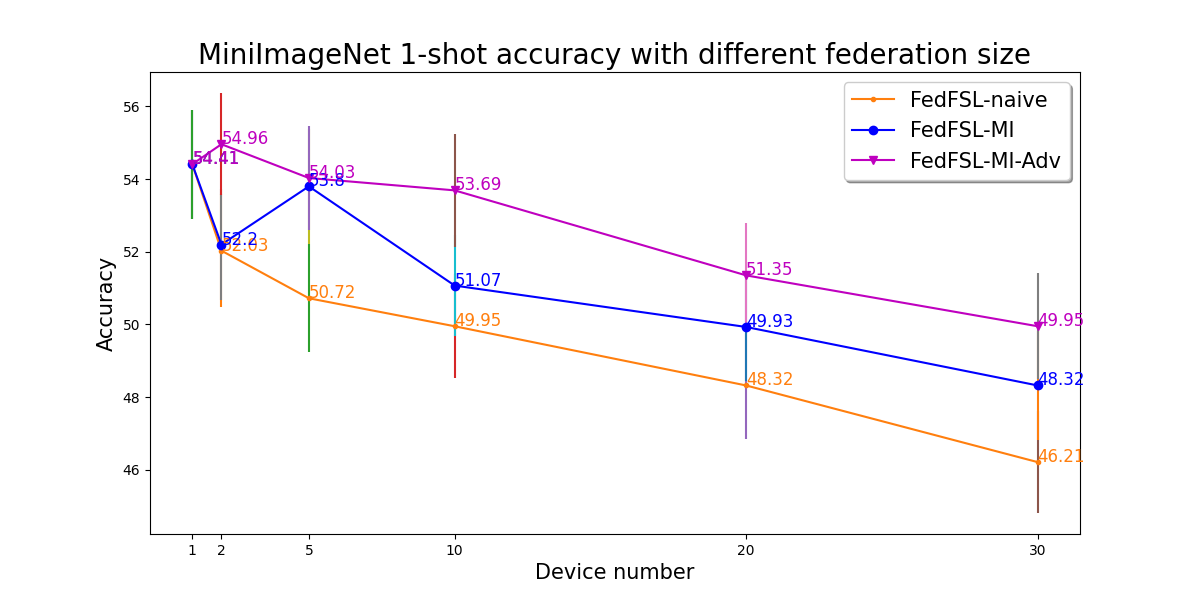}
\end{subfigure}
\vspace{-8pt}
\caption{FSL accuracy of 1-shot task on MiniImageNet w.r.t. number of devices in federation.}
\vspace{-8pt}
\label{fig:device_number}
\end{figure}

%\makeatletter\def\@captype{figure}\makeatother
%\begin{minipage}{0.22\textwidth}
%\centering
%\hspace{-30pt}
%\includegraphics[scale=0.28, trim=0 0 20 0]{figures/data_dist_device_1.png_train.png}
%\vspace{-2pt}
%\caption{Non-IID data.}
%\label{fig:data_nid}
%\end{minipage}
%\makeatletter\def\@captype{table}\makeatother
%\begin{minipage}{.26\textwidth}
%\centering
%\small
%\begin{tabular}{lcc} \toprule  
%\multirow{2}{*}{\textbf{Method}} & \multicolumn{2}{c}{\textbf{Non-IID}} \\ 
%\cmidrule{2-3}     & 1-shot  & 5-shot \\ \toprule
%FedFSL-local &  66.68\% & \%	\\
%FedFSL-naive &  68.85\% & 70.62\%	\\
%FedFSL-naive &  70.77\% & 72.25\%	\\
%FedFSL-MI  &  70.37\% &  73.25\% \\
%FedFSL-MI-Adv  & \textbf{71.35\%} & \textbf{ 76.00\%}   \\  \midrule \bottomrule
%\end{tabular}
%\vspace{-2pt}
%\caption{Sent140 results.}
%\label{tab:res_sent140}
%\end{minipage}

\subsubsection{Results on Text Classification}
In Table~\ref{tab:res_sent140}, we compare our methods on Sent140 dataset with 1-shot / 5-shot tasks learned by a federation of 5 clients. 
Following the provided tool of partitioning the dataset, we distribute different users' data to each client without replacement. Since the data distributions vary for users, this sampling process provides non-IID data partitions. Our goal is to train an effective global sentiment classification model on one portion of users, which can be used to detect the sentiment on disjoint new users. This is particularly challenging because different users can use very distinct words and exclamations to express feelings, and many of them are rare.

In this task, our backbone model is a GRU (RNN) network with hidden size 128. We convert tweet sentences to sequences of 300-D GloVe~\cite{pennington2014glove} word vectors as input to the GRU model. We add a binary classifier upon GRU hidden output to classify negative and positive sentiment.
We examine the performance of baselines and our models and observe that
\begin{itemize}[leftmargin=*]
\item FedFSL-MI-Adv outperforms the other approaches in this natural language understanding task, similar as in image classification task. It also shows that our FedFSL framework can be applicable to both CNN and RNN models.
\item We found that the performance increases from 1-shot to 5-shot tasks are generally less than image classification tasks, e.g., less than 5\% on Sent140 while more than 10\% on miniImageNet. This is because the few-shot labelled sentences can only provide a few more words to help adapt to a user's emotion, while images can provide much richer details and patterns of a given object.
\end{itemize}

\subsubsection{Different device number}
We study the trend of accuracy of FedFSL with different number of participating devices $K=2,5,10,20,30$, as well as $K=1$ to simulate complete centralized training.
We illustrate the results for non-IID 5-way 1-shot task with miniImageNet  in Fig.~\ref{fig:device_number} with detailed numbers.
Note that the more participating devices, the fewer training samples each device holds.
% as we uniformly partition the data. For miniImageNet, each data class has 600 samples in total. With different $K$, there are $\lfloor \frac{600}{K} \rfloor$ images per class for each device, which ranges
%from 20 (when $K=30$) to 300 (when $K=2$).  
We observe that
\begin{itemize}[leftmargin=*]
\item The overall trend is that more participating devices yielded decreased accuracy for all 3 approaches. The task becomes more difficult when $K$ increases as the device coordination grows harder and the client model becomes less capable with less training data.
\item \emph{FedFSL-MI-Adv still achieves the best results} on all cases, leading the second best FedFSL-MI by more 2-5\% relatively.
\item The performance of \emph{FedFSL-MI-Adv decreases more slowly than other approaches} with the increase of $K$, which indicates the beneficial of learning a consistent feature space over the clients.
%\item \emph{Centralized training does not significantly outperforms distributed training.} 
\item \emph{FedFSL-MI-Adv in 2-device federation works even better than 1-device centralized training}, with accuracy 54.96\% v.s. 54.41\%. Note that the total training samples of each device get halved on each device when $K=2$. The surprising result that distributed training outperforms centralized training can be explained that FSL is aiming to learn with very few training samples, instead of fitting a task with many samples as in supervised learning. Therefore, FSL is relative less sensitive to the number of examples in \emph{base} classes on each client. Moreover, by utilizing our approach to align decision boundaries well, the two client models form an effective ensemble to enhance the overall performance, compared with a single-model case.
\end{itemize}

\subsubsection{Mutual information with $k$-exclusive global model}
\label{sec:ab_mi}
In Section~\ref{sec:fed_maml_mi}, we discussed that we can estimate the mutual information produced by $w_{-k}^t$ with the global model $w^t$ when the number of devices are large. We conducted experiments to compare using the original $w_{-k}^t$ with $w^t$ in (\ref{eq:fed_maml_mi}).
Surprisingly, we found that using $w^t$ either outperforms or ties with using $w_{-k}^t$ as original definition in both 1-/5-shot tasks in all datasets, with the only exception of miniImageNet 5-shot task (69.61\% v.s. 70.70\%). 
We found that this phenomenon could be explained by``self-knowledge distillation" such that a trained model can be used to improve itself~\cite{yun2020regularizing, zhang2020selfdistillation} as it provides a kind of label smoothing which can make training processes robust to noises. 

\section{Conclusion}
In this paper, we proposed a framework that makes federated learning effective in data-scarce scenarios. 
We designed an adversarial learning strategy to 
construct a consistent feature space over the clients, to better learn from scarce data.
Experimental results show that our adversarial learning based method outperforms baseline methods by 5\%$\sim$15\% on benchmark datasets.
Future work can investigate the
theoretical convergence analysis of FedFSL with non-convex models and how to further extend FedFSL to regression and reinforcement learning tasks.

\bibliographystyle{ACM-Reference-Format}
\bibliography{ref}

\end{document}